\documentclass[sigconf]{acmart}
\setcopyright{none}
\renewcommand\footnotetextcopyrightpermission[1]{}
\usepackage{algorithm, algorithmic, amsmath, bm, subcaption, cleveref, multirow, makecell, mathtools, graphicx, bbm, rotating, balance}
\usepackage{thmtools, thm-restate}
\usepackage[normalem]{ulem}
\usepackage{enumitem}
\usepackage{tcolorbox}
\usepackage[table]{xcolor}
\useunder{\uline}{\ul}{}

\AtBeginDocument{%
  }

\acmConference[Conference acronym 'XX]{Make sure to enter the correct
  conference title from your rights confirmation emai}{June 03--05,
  2018}{Woodstock, NY}
\acmISBN{978-1-4503-XXXX-X/18/06}

\newcommand{\beqa}{\begin{eqnarray}}
\newcommand{\eeqa}{\end{eqnarray}}
\newcommand{\beq}{\begin{equation}}
\newcommand{\eeq}{\end{equation}}
\newcommand{\ben}{\begin{enumerate}}
\newcommand{\een}{\end{enumerate}}
\newcommand{\bit}{\begin{itemize}}
\newcommand{\eit}{\end{itemize}}
\newcommand{\bi}{\begin{itemize} \item}
\newcommand{\ei}{\end{itemize}}

\newcommand{\begindef}{\begin{Definition} \rm}
\newcommand{\beginexa}{\begin{Example} \rm}
\newcommand{\beginthe}{\begin{Theorem} \rm}
\newcommand{\beginpro}{\begin{Proposition} \rm}
\newcommand{\beginlem}{\begin{Lemma} \rm}
\newcommand{\begincon}{\begin{Conjecture} \rm}
\newcommand{\begincor}{\begin{Corollary} \rm}

\newcommand{\eat}[1]{}

\renewcommand{\paragraph}[1]{\noindent\textbf{#1.}}

\def\papernumber #1 raised #2 {
\vspace{-#2}
\vbox to 0pt{\hfill\framebox{\bf Paper Number #1}}
\vspace{#2}
}

\def\T{{\scriptscriptstyle\mathsf{T}}}

\def\s{\bm{s}}
\def\S{\bm{S}}
\def\Wl{\bm{W}^{(l)}}
\def\uli{\bm{u}^{(l)}_i}
\def\vli{\bm{v}^{(l)}_i}
\def\Ul{\bm{U}^{(l)}}
\def\Vl{\bm{V}^{(l)}}

\def\U{\mathcal{U}}
\def\I{\mathcal{I}}
\def\S{\mathcal{S}}

\def\q{\bm{q}}

\newcommand{\bluecell}[1]{\cellcolor[HTML]{CFE2F3}#1}

\newcommand{\redcell}[1]{\cellcolor[HTML]{FADBD8}#1}

\def\algname{\textsc{HiLoMoE}}
\begin{document}

\title{Hierarchical LoRA MoE for Efficient CTR Model Scaling}


\author{Zhichen Zeng$^1$\, Mengyue Hang$^2$\, Xiaolong Liu$^2$\, Xiaoyi Liu$^2$\, Xiao Lin$^1$\, Ruizhong Qiu$^1$\, Tianxin Wei$^1$\, Zhining Liu$^1$\, Siyang Yuan$^2$\, Chaofei Yang$^2$\, Yiqun Liu$^2$\, Hang Yin$^2$\, Jiyan Yang$^2$\, Hanghang Tong$^1$\\
$^1$ University of Illinois Urbana-Champaign, $^2$ Meta AI\\
\small\texttt{zhichenz@illinois.edu, hangm@meta.com}}
\renewcommand{\shortauthors}{Zhichen Zeng et al.}

\begin{abstract}
    Deep models have driven significant advances in click‑through rate (CTR) prediction.
    While \emph{vertical scaling} via layer stacking improves model expressiveness, the layer-by-layer sequential computation poses challenges to efficient scaling.
    Conversely, \emph{horizontal scaling} through Mixture of Experts (MoE) achieves efficient scaling by activating a small subset of experts in parallel, but flat MoE layers may struggle to capture the hierarchical structure inherent in recommendation tasks.
    To push the Return-On-Investment (ROI) boundary, we explore the complementary strengths of both directions and propose \algname, a hierarchical LoRA MoE framework that enables holistic scaling in a parameter-efficient manner.
    Specifically, \algname\ employs lightweight rank-1 experts for parameter-efficient horizontal scaling, and stacks multiple MoE layers with hierarchical routing to enable combinatorially diverse expert compositions. Unlike conventional stacking, \algname\ routes based on prior layer scores rather than outputs, allowing all layers to execute in parallel.
    A principled three-stage training framework ensures stable optimization and expert diversity. 
    Experiments on four public datasets show that \algname\ achieving better performance-efficiency tradeoff, achieving an average AUC improvement of 0.20\% in AUC and 18.5\% reduction in FLOPs compared to the non-MoE baseline.
\end{abstract}


\received{20 February 2007}
\received[revised]{12 March 2009}
\received[accepted]{5 June 2009}

\maketitle

\section{Introduction}\label{sec:intro}

\begin{figure}[t]
    \centering
    \includegraphics[width=.9\linewidth]{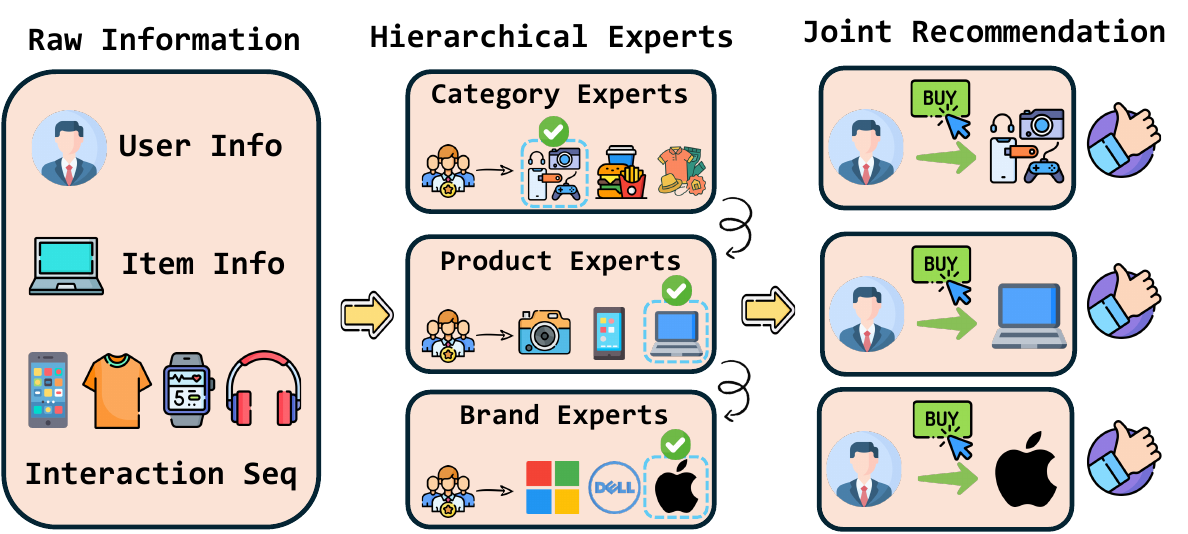}
    \caption{Motivation for Hierarchical MoE. 
    Recommendation follows a hierarchical structure, where user–item interactions span multiple granularities.
    By hierarchically selecting experts such as \texttt{Electronics} (category), \texttt{Laptop} (product), and \texttt{Apple} (brand), the model collaboratively delivers personalized recommendations at different levels, effectively leveraging hierarchical knowledge for more accurate predictions.}
    \vspace{-15pt}
    \label{fig:teaser}
\end{figure}
Click-through rate (CTR) prediction serves as a foundational task in modern recommender systems, enabling platforms to rank items, allocate ad impressions, and personalize content delivery~\cite{lin2024clickprompt,liu2024collaborative,yoo2024ensuring, DBLP:conf/kdd/LiFAH25,DBLP:conf/nips/BanZLQFKTH24}.
As the user and system scale continue to grow, there has been an increasing demand for more expressive and efficient CTR models that can make accurate predictions under tight latency and resource constraints~\cite{borisyuk2024lirank,xia2023transact,zeng2023parrot,zeng2023generative,zeng2024graph,zeng2024hierarchical,zeng2025pave}.
To meet this demand, \emph{model scaling}, the ability to grow model capacity in a performance- and cost-effective manner, emerges as a critical challenge~\cite{zhang2022dhen,zhang2024wukong,liang2025external,zhang2025balancing}.

In recent years, we have observed extensive efforts on \emph{vertical scaling} via layer stacking~\cite{zhang2022dhen,zhang2024wukong,zeng2024interformer}, where deeper networks expand model capacity and allow for more complex representations.
Despite noticeable performance gains, vertical scaling inherently requires sequential computation, leading to substantial inference latency as depth increases.
In addition, stacking more layers often introduces a large number of parameters, making vertical scaling less suitable for resource-constrained scenarios where parameter efficiency is critical.
Moreover, vertical scaling lacks conditional computation as every input passes through the full stack of layers, leaving little room for personalized computation that is crucial for CTR prediction where user interests and contexts vary widely.

To enable more efficient model scaling, recent works explore \emph{horizontal scaling} based on the mixture of experts (MoE) framework, ~\cite{fedus2022switch,zoph2022stmoe} and have achieved great success in large language models~\cite{jiang2024mixtral,liu2024deepseek}.
MoE introduces multiple experts and dynamically selects a \emph{small subset} for each input, enabling personalized computation while significantly reducing inference cost and memory usage.
However, in CTR prediction, both user and item information often exhibit a strong hierarchical structure that flat, single-layer MoE may fail to capture~\cite{zheng2022hien}.
As shown in Figure~\ref{fig:teaser}, item features typically include high-level category information (e.g., electronics or apparel), followed by mid-level product information (e.g., camera or laptop), and finally finer-grained brand identity (e.g., Apple or Microsoft).
Such multi-level semantics are difficult to disentangle within a single-layer MoE, often leading to under-specialized experts and limited expressiveness.
This motivates the need for \emph{hierarchically-structured expert designs} that can better model the compositional nature of inputs and promote expert diversity.

To address these challenges, it is essential to explore complementary roles of vertical and horizontal scaling.
In this paper, we propose \algname, a hierarchical LoRA MoE framework that enables \emph{holistic model scaling} in both vertical and horizontal directions in a parameter-efficient manner, which is built upon three key innovations.
First (\emph{LoRA experts}), to support efficient horizontal scaling, we represent each expert as a rank-1 perturbation of a shared base weight matrix, which significantly reduces parameter and memory overhead, allowing us to scale the number of experts without inflating the overall model size.
Second (\emph{hierarchical routing}), to achieve efficient vertical scaling, we introduce a hierarchical routing mechanism, where the expert selection at each layer conditions on the routing scores from previous layers.
For one thing, such hierarchically-structured design enables combinatorially diverse expert paths, boosting model expressiveness without increasing per-layer width.
For another, the conditional routing computation depends only on routing scores rather than the intermediate expert outputs, enabling parallel inference across layers.
Third (\emph{training framework}), to train the complex system, we propose a three-stage training pipeline, augmented with auxiliary losses and controlled gradient flow, to stabilize training and promote expert diversity.

The main contributions of this paper are summarized as follows:
\begin{itemize}
    \item \textbf{Model design.} We propose a hierarchical LoRA MoE framework to achieve efficient scaling for Transformer-based CTR models. The LoRA experts support efficient horizontal scaling for personalized computation, while the hierarchical routing mechanism enables vertical scaling through deep expert compositions with parallelizable inference.
    \item \textbf{Training framework.} We introduce a principled training framework, consisting of a three-stage training pipeline for stable optimization, and auxiliary losses to ensure balanced expert utilization and effective specialization.
    \item \textbf{Experiments.} Benchmark experiments suggest that \algname\ consistently enhances Transformer-based models by an average of 0.20\% in AUC and 18.5\% reduction in FLOPs compared to non-MoE counterparts. Besides, \algname\ exhibits promising scaling in depth and width, with performance increasing alongside the number of layers and experts.
\end{itemize}

\vspace{-5pt}
\section{Related Works}\label{sec:related}

In the era of big data and AI~\cite{lin2025quantization,lin2025toklip,lin2024duquant,li2025flow}, improving recommendation quality while controlling model complexity is a core challenge~\cite{liu2023class,zheng2024pyg,jiang2025image}, since even minor gains in prediction accuracy brings substantial improvements in revenue and user experience~\cite{xu2024slog,xu2024discrete,yan2024pacer,yan2024thegcn,yan2021bright,yan2023reconciling,yu2025joint}. 
In this section, we review related works on CTR prediction, MoE and LoRA as the foundation to understand \algname.

\vspace{-5pt}
\subsection{CTR Prediction Models}
Click‑Through Rate (CTR) prediction estimates the likelihood of a user clicking an item, an essential component in recommender systems.
Early models~\cite{cheng2016wide,guo2017deepfm,lian2018xdeepfm,sun2021fm2,wang2021dcn,yang2017bridging,zhou2020can} follow the factorization machine paradigm, combining memorization and generalization by fusing linear models with deep neural networks.
With richer user sequences, Transformer‑based models~\cite{sun2019bert4rec,lyu2020deep,lin2024clickprompt,lin2024backtime,lin2025cats} become prominent for capturing sequential dependencies and complex interactions.
Models like BST employ Transformer encoders for behavior sequence modeling~\cite{chen2019behavior}, and DIN~\cite{zhou2018deep} and DIEN~\cite{zhou2019deep} employed attention mechanisms and evolving interest to dynamically focus on relevant user actions.
Building on these foundations, TransAct~\cite{xia2023transact} blends real‑time user actions with long‑term interests via a hybrid real‑time/batch modeling framework, LiRank~\cite{borisyuk2024lirank} employs ensembles of multiple interaction modules in large‑scale ranking systems, CARL~\cite{chen2024cache} accelerates computation through cache‑aware mechanisms, END4Rec~\cite{han2024end4rec} introduces an efficient miner and denoising modules for multi‑behavior sequences, SRP4CTR~\cite{han2024enhancing} bridges pre‑trained sequential encoders with CTR models via cross‑attention, and InterFormer~\cite{zeng2024interformer} enhances cross‑modal integration via interleaved information flow.

\vspace{-5pt}
\subsection{MoE and LoRA}
To achieve efficient model scaling, mixture of experts and low-rank adaptation have served as prominent approaches.

MoE architecture increases model capacity by dynamically selecting a small subset of experts for each input, enabling personalized computation while maintaining inference efficiency.
Early works~\cite{jacobs1991adaptive,nowlan1990evaluation,jordan1994hierarchical,ai2025resmoe} utilize MoE to decompose the large task into smaller ones to enhance efficiency.
Recently, MoE has attracted extensive attention due to its success in efficient scaling in large language models~\cite{liu2024deepseek,jiang2024mixtral,fedus2022switch, du2022glam}.
Sparsely-gated MoE layers~\cite{shazeer2017outrageously,lepikhin2020gshard,fedus2022switch} achieve significant improvement in model capacity with minor efficiency loss.
Auxiliary losses are further proposed to achieve stable model training~\cite{zoph2022stmoe} and balanced expert loading~\cite{fedus2022switch}.

In parallel, LoRA~\cite{hu2021lora,zhang2023composing,liu2024dora, lin2025cats} enables efficient model adaptation by injecting low-rank parameter matrices into frozen networks, delivering full fine-tuning performance with far fewer trainable parameters.
Modular strategies~\cite{pfeiffer2020adapterfusion,huang2023lorahub,sung2022vl} utilize LoRA for multi-task learning through dynamic module fusion.
To reduce redundancy and parameter overhead, different works adopt weight tying~\cite{renduchintala2023tied}, random adaptation~\cite{kopiczko2023vera,bershatsky2024lotr}, and federated-efficient designs~\cite{hyeon2021fedpara}.

Recently, MoE and LoRA have converged in hybrid designs to unlock richer adaptability and task efficiency.
Early examples~\cite{wang2022adamix,gou2023mixture} blend multiple adaptation modules in each Transformer layer using MoE-style routing to boost performance without additional inference.
More advanced designs include task-aware routing~\cite{liu2023moelora,liu2024adamole}, asymmetric LoRA experts~\cite{tian2024hydralora,wang2024malora}, and composite expert architectures for better performance–efficiency tradeoffs~\cite{li2024mixlora,wu2024mixture,zeng2025smore}.

\subsection{MoE on CTR Models}
Recent CTR models leverage MoE frameworks to enhance scalability and context sensitivity.
Mixture of LoRA~\cite{yang2024mlora,gao2025mlora+,yagel2025moe} embed domain-specific adapters as experts and route inputs dynamically.
Multi-task and multi-domain modeling efforts~\cite{zhang2024m3oe,ma2018modeling} decompose knowledge into shared, domain-specific, and task-specific experts, enabling flexible cross-domain personalized adaptation.
Personalization and modality-aware routing are explore~\cite{li2025rankexpert,bian2023multi,zhu2020recommendation,xu2024mome,nguyen2025multi}, mixing textual, behavioral, auxiliary, and modality-biased experts for diverse user contexts.
Hierarchical MoE architectures~\cite{tang2020progressive,li2023adatt,zhang2025hierarchical} use multi-level gating structures to sequentially integrate shared, task-specific, and temporal information.
To encourage expert diversity, D-MoE~\cite{wang2025enhancing} introduces a cross-expert de-correlation loss.
\section{Preliminaries}\label{sec:prelim}
We use bold uppercase letters for matrices (e.g., $\bm{X}$), bold lowercase letters for vectors (e.g., $\bm{x}$), and lowercase letters for scalars (e.g., $n$).
We use superscripts to denote layer index, e.g., $\bm{u}^{(l)}$ and subscripts to denote expert index, e.g., $\bm{u}_{i}$.
User set and item set are denoted by $\U$ and $\I$, respectively.
The interaction sequence of a user $u\in\U$ is denoted as $S_u=[i_{u,1}, \dots, i_{u,T}]\in\S$, where $i_{u,t}\in\I$.
We use $y_{i_{u,t}}\in\{0,1\}$ to indicate whether the user $u$ clicked on the item $i_{u,t}$ at timestamp $t$.

\paragraph{CTR prediction}
CTR prediction estimates the probability of a user clicking on an item given heterogeneous information, e.g., sparse categorical features, numerical features, and interaction sequences.
The goal of CTR prediction is to learn a function $f\colon\U\times\I\times\S\to[0,1]$ such that $\Pr(y_{i_{u,T+1}}\!=\!1 \mid u, i_{u,T+1}, S_u; \theta) = f(u, i_{u,T+1}, S_u; \theta)$ with high accuracy.

\paragraph{Mixture of Experts}
An MoE layer consists of $K$ experts $\{E_k\}_{k=1}^K$ and a gating function $g\colon \mathbb{R}^{d}\to\Delta^{K-1}$ that outputs probabilities over experts.  Given an input $h\in\mathbb{R}^{d}$, the output is a weighted sum of expert outputs:
\begin{equation*}
  \mathrm{MoE}(h) = \sum_{k=1}^K g_k(h) E_k(h).
\end{equation*}

\paragraph{Low‑Rank Adaptation}
LoRA introduces a pair of low‑rank matrices $(\bm{A},\bm{B})$ into a linear
layer $\bm{W}\in\mathbb{R}^{d_{\text{out}}\times d_{\text{in}}}$ by
parameterizing
\begin{equation*}
  \bm{W}' = \bm{W} + \bm{BA},
\end{equation*}
where $\bm{A}\in\mathbb{R}^{r\times d_{\text{in}}}$ and $\bm{B}\in\mathbb{R}^{d_{\text{out}}\times r}$ with rank $r\ll \min(d_{\text{in}},d_{\text{out}})$.
In general, $\bm{W}$ is the pre-trained model parameter, which remains frozen during fine-tuning, and $\bm{A},\bm{B}$ are learnable parameters for fine-tuning.
This reduces trainable parameters to $O(r(d_{\text{in}} + d_{\text{out}}))$ and allows efficient storage of multiple task adapters.
\section{Methodology}\label{sec:method}

In this section, we introduce our proposed \algname, a hierarchical LoRA mixture of experts architecture for holistic model scaling.
We first provide an overview of the model architecture in Section~\ref{sec:overview}.
Afterwards, we introduce the parameter-efficient LoRA expert design in Section~\ref{sec:lora}, followed by the hierarchical routing mechanism in Section~\ref{sec:routing}.
A principled training framework is further introduced in Section~\ref{sec:train}.

\subsection{Model Overview}\label{sec:overview}
\begin{figure}[t]
    \centering
    \includegraphics[width=\linewidth]{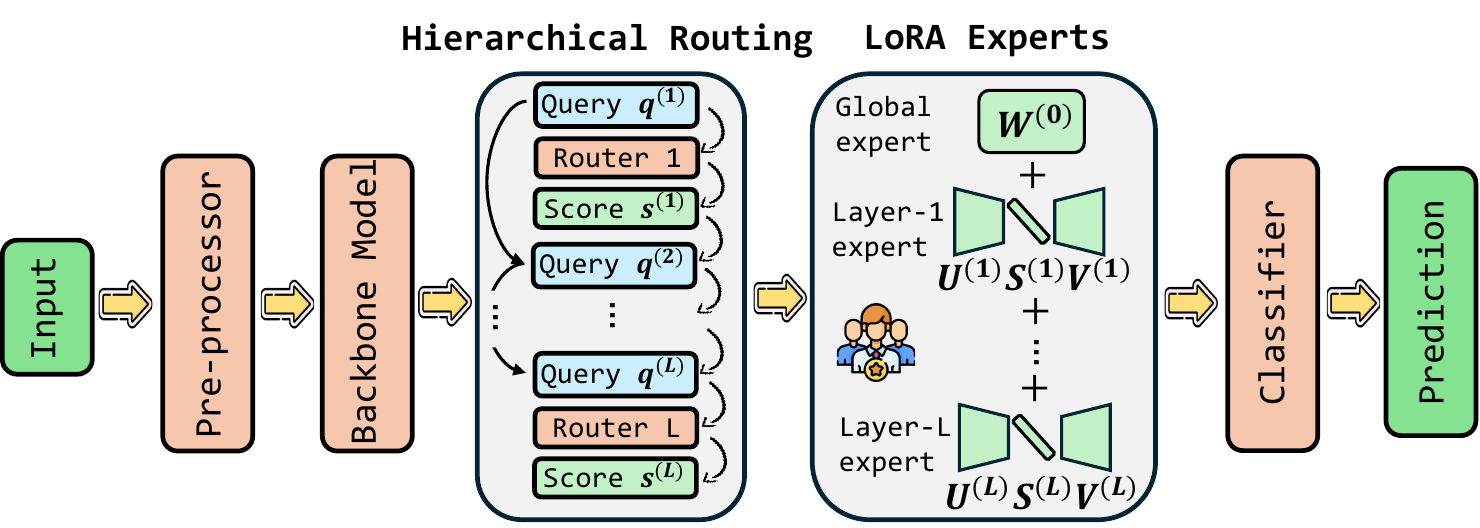}
    \vspace{-10pt}
    \caption{An overview of the proposed \algname\, which includes $L$ layers, each containing $K$ rank-1 experts.}
    \vspace{-10pt}
    \label{fig:model}
\end{figure}

Figure~\ref{fig:model} presents an overview of the proposed \algname\ framework, which is designed to achieve parameter-efficient and scalable modeling through a combination of \emph{LoRA experts} and a \emph{hierarchical routing mechanism}.

First, to enable efficient horizontal scaling, \algname\ employs LoRA experts, where each expert is implemented as a rank-1 low-rank approximation of a full-rank weight matrix.
This reduces the parameter complexity of each expert from quadratic to linear with respect to hidden size, making the model lightweight yet expressive.

Second, to facilitate efficient vertical scaling, \algname\ introduces a hierarchical routing mechanism that enables expert selection at each layer based on both the input query and the routing decisions from previous layers.
Unlike conventional layer stacking where each layer is computed sequentially, \algname\ decouples lightweight routing score computation from computationally heavy expert execution.
Specifically, we compute the lightweight routing scores layer-by-layer in a hierarchical fashion, but execute the heavy expert transformation computations \emph{in parallel} across all layers, leading to significant efficiency gains during inference.

Together, these two design choices make \algname\ a compelling solution for scalable CTR prediction models that balance expressiveness, efficiency, and modularity.

\subsection{Parameter-Efficient LoRA experts}\label{sec:lora}
We first introduce our parameter-efficient LoRA experts.
LoRA has been widely adopted in multi-domain settings, where a shared base weight serves as the foundation for capturing generalizable knowledge, while lightweight low-rank adapters specialize in domain-specific nuances.
The low-rank formulation not only improves efficiency by reusing the heavy base weights across different adapters, but also enhances robustness: the shared backbone provides stable, broadly applicable representations, while the low-rank updates enable fine-grained adaptation to each domain.

Given that CTR models often operate under tight FLOPs and latency budgets, rank-1 experts are adopted to achieve extreme parameter efficiency for high-frequency inference.
Specifically, for the $l$-th layer, each expert is the product of two vectors $\Wl_i=\uli{\vli}^\T,\forall i=1,2,...,K$, where $\uli,\vli\in\mathbb{R}^{d}$. Given a routing score $\bm{s}^{(l)}\in\Delta^{K-1}$, the $l$-th layer expert weight $\Wl$ is defined as
\begin{equation*}
    \Wl=\sum_{i=1}^K\bm{s}^{(l)}_i\Wl_i=\Ul\text{diag}\left(\bm{s}^{(l)}\right){\Vl},
\end{equation*}
where $\Ul=\left[\bm{u}_1^{(l)}\|\cdots\|\bm{u}_K^{(l)}\right]\in\mathbb{R}^{d\times K}, \Vl=\left[\bm{v}_1^{(l)}\|\cdots\|\bm{v}_K^{(l)}\right]^\T\in\mathbb{R}^{K\times d}$ are the concatenations of expert weights, and $\bm{W}^{(0)}$ is the shared expert weight. And the final transformation weight is the addition of the expert weights at each layer and a shared base weight $\bm{W}^{(0)}$, that is
\begin{equation}\label{eq:lora}
    \bm{W} = \bm{W}^{(0)} + \sum_{l=1}^L\Ul\text{diag}\left(\bm{s}^{(l)}\right){\Vl}.
\end{equation}

During inference, the expert weights $\Ul,\Vl$ at each layer remain fixed, and the final transformation weight is determined solely by the layer-wise routing scores $\bm{s}^{(l)}$.
As we will detail in the next section, the computation of routing scores is decoupled from the outputs of the MoE layers and instead depends only on the routing scores from preceding layers.
This decoupling enables us to pre-compute all routing scores in a lightweight forward pass.
Subsequently, the expert contributions across all layers can be aggregated and the sequence transformation applied in a single operation.
As a result, \emph{vertical scaling incurs no additional inference cost}, since the heavy expert computations across layers can be executed in parallel after routing is resolved.

\subsection{Hierarchical Routing Strategy}\label{sec:routing}
\begin{figure}[t]
    \centering
    \includegraphics[width=.9\linewidth]{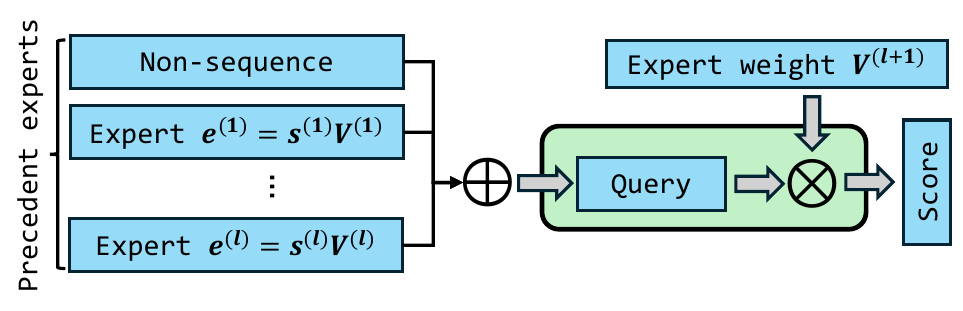}
    \vspace{-10pt}
    \caption{A hierarchical routing strategy selects the current experts based on selections from the previous layers.}
    \vspace{-10pt}
    \label{fig:router}
\end{figure}

Unlike conventional MoE frameworks that primarily scale in width by increasing the number of experts, \algname\ introduces a new paradigm that enables simultaneous depth–width scaling.
This design enables the model to benefit from deep hierarchical transformations (depth) and combinatorially diverse expert selections (width), enhancing expressiveness and scalability.

To this end, our router design is guided by the following three principles.
First (\emph{hierarchical routing}), The routing decision at each layer is conditioned on the routing outcomes from preceding layers, forming a hierarchical selection process. This structure encourages the model to extract information in a coarse-to-fine manner. For example, in CTR prediction tasks, earlier layers may focus on high-level attributes such as item categories, while later layers refine these decisions based on fine-grained features like item brand or user intent.
Second (\emph{parallel inference}), to ensure scalability without incurring additional inference cost, different MoE layers are designed to execute in parallel. 
Third (\emph{heterogeneous interaction}), CTR models often receive heterogeneous information, e.g., sequence and non-sequence information, and effective interaction among these features is at the core of the success of CTR prediction~\cite{zeng2024interformer}.

To meet these principles, we introduce the hierarchical routing design shown in Figure~\ref{fig:router}.
Specifically, non-sequence information is processed by a query projection matrix $\bm{W}_{\text{proj}}$ to extract a low-dimensional representation $\bm{x}\in\mathbb{R}^d$ as the initial query $\q^{(1)}$.

For $l$-th layer router, it takes the input query $\q^{(l)}$ as input and calculates the routing score by the inner product of input query and expert weights $\Vl$, that is
\begin{equation*}
    \bm{s}^{(l)}=\text{Softmax}\left(\bm{q}^{(l)}\Vl/\sqrt{d}\right).
\end{equation*}

To enable hierarchical routing, the key is to design a proper query $\q^{(l)}$ that incorporates expert information from previous layers.
For each layer, we compute the expert representation $\bm{e}^{(l)}$ using the expert weights $\Vl$ averaged by the routing score $\s^{(l)}$, i.e., $\bm{e}^{(l)} = \s^{(l)}{\Vl}$.
The input query $\q^{(l+1)}$ for next layer is further updated by the sum of previous query $\q^{(l+1)}$ and expert representation $\bm{e}^{(l)}$:
\begin{equation}\label{eq:query}
    \q^{(l+1)}=\q^{(l)}+\bm{e}^{(l)}=\bm{x} + \sum_{i=1}^{l}\bm{e}^{(i)}.
\end{equation}
Such query update and routing score computation bring two notable advantages.
For one thing, hierarchical routing enables structured, layer-dependent expert selection, which encourages combinatorially diverse expert compositions across layers and enhances model expressiveness.
For another, the routing score computation is decoupled from the actual expert transformation, allowing the full set of selected expert transformations to be executed simultaneously after routing scores are determined, which substantially reduces the computation overhead typically associated with vertical scaling.

\subsection{Principled Training Framework}\label{sec:train}
\begin{figure}[t]
    \centering
    \includegraphics[width=\linewidth]{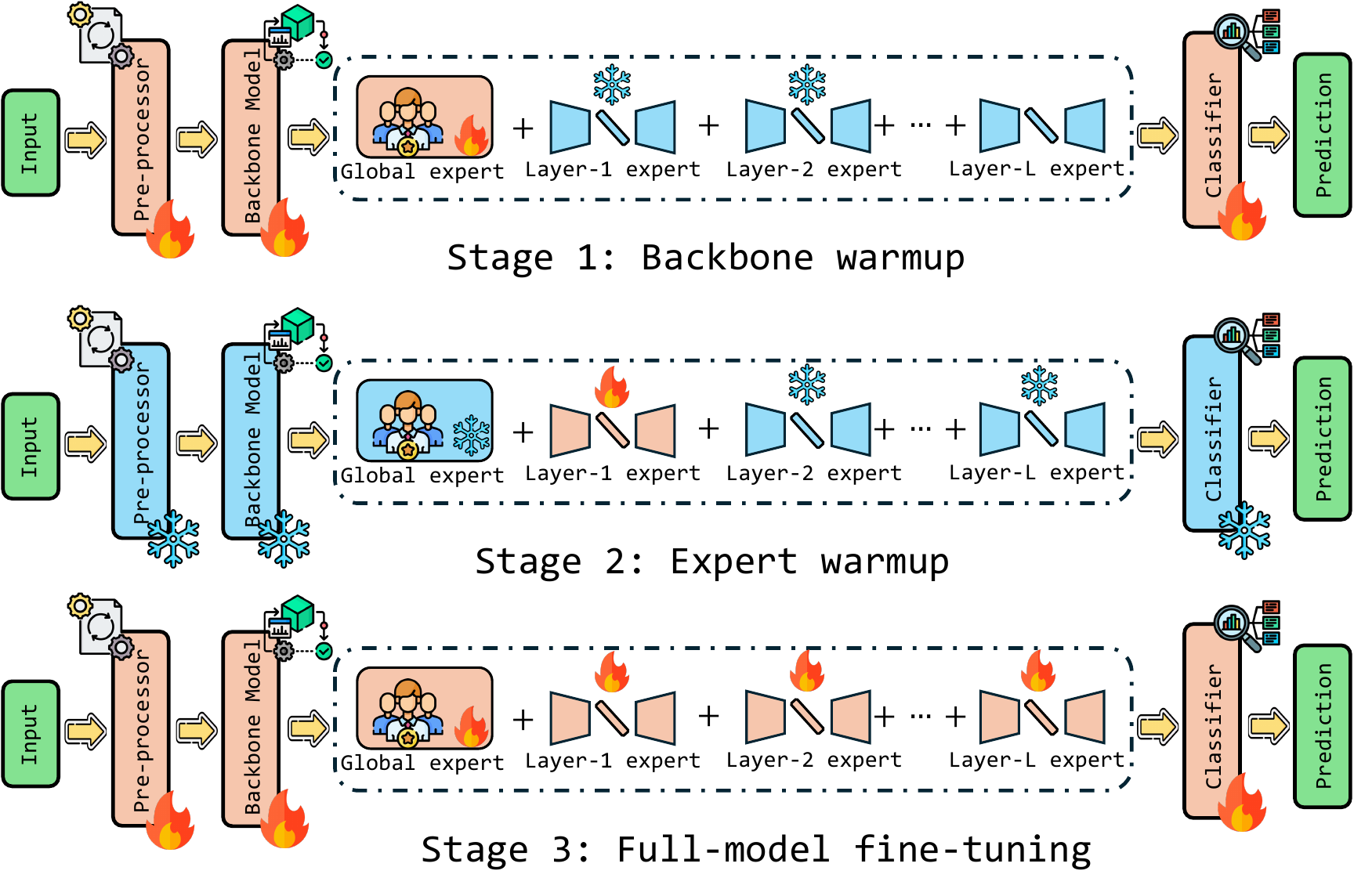}
    \vspace{-20pt}
    \caption{A 3-stage training framework for stable training.}
    \vspace{-15pt}
    \label{fig:train}
\end{figure}
Apart from model design, a principled training framework is essential to enable stable and effective model scaling.
We address this challenge from two aspects, including a progressive warmup strategy to stabilize training, and auxiliary losses to promote balanced and diverse expert utilization.

\paragraph{Model Warmup}
Directly training the full model can be unstable due to the intricate dependencies across layers and experts
To mitigate this, we adopt a 3-stage training pipeline shown in Figure~\ref{fig:train} to progressively build the model in a bottom-up fashion.

In the first stage (\emph{backbone warmup}), we train only the backbone components,including the pre-processor, backbone model, classifier, and the shared expert.
This stage is equivalent to training a \emph{non-MoE baseline model} and serves as a stable initialization for the subsequent stages.

In the second stage (\emph{expert warmup}), we sequentially activate and train the LoRA experts layer-by-layer, while freezing the backbone and previously activated experts.
To avoid interference from untrained experts, experts are initialized with zero-valued LoRA weights, i.e., $\Vl\sim\mathcal{N}(0,I),\Ul\leftarrow\bm{0}$.
This ensures that each layer is trained as an \emph{incremental refinement} over earlier representations, leading to more stable optimization and faster convergence.

In the final stage (\emph{full-model fine-tuning}), we fine-tune the entire model jointly, updating all parameters from the well-initialized starting point provided by the previous stages.
This final step ensures global consistency and maximizes model performance.

\paragraph{Auxiliary Losses}
Another critical challenge is to ensure balanced expert utilization.
Without appropriate load-balancing mechanisms, the router may consistently favor a small subset of experts, while leaving others underutilized.
This phenomenon, known as expert collapse~\cite{fedus2022switch,chi2022representation}, undermines the intended modularity of MoE by effectively reducing it to a single-expert non-MoE architecture, 
thereby limiting its capacity and scalability.

To address this issue, we adopt two auxiliary losses, including load-balancing loss $\ell_{lb}$~\cite{fedus2022switch} and z-loss $\ell_z$~\cite{zoph2022stmoe}, defined as follows
\begin{equation}\label{eq:aux}
    \begin{aligned}
        &\text{load-balancing loss}: \ell_{lb} = \frac{N}{B\cdot A} \sum_{b,k} \bm{s}(b,k) \cdot \mathbbm{1}_{b,k},\\
        &\text{z-loss}: \ell_z = \text{logsumexp}(\bm{s}),
    \end{aligned}
\end{equation}
where $B$ is the batch size, $A$ is the number of activated experts, and $\mathbbm{1}_{b,k}$ is the indicator function denoting whether expert $k$ is activated for sample $b$.
The load-balancing loss encourages a uniform distribution of routing scores by penalizing over-reliance on a subset of experts, promoting balanced expert utilization. 
The z-loss regularizes the scale of gating logits, preventing excessively large values that lead to numerical instability and degraded training dynamics.

Although the auxiliary losses in Eq.~\eqref{eq:aux} are designed to encourage diverse and balanced expert utilization, improper use may lead to conflicts with the main prediction objective, ultimately degrading model performance.
To mitigate this risk, we isolate the influence of the auxiliary losses by \emph{restricting their gradient flow}.
Specifically, we allow gradients from $\ell_{lb}$ and $\ell_z$ to update only the router parameters, while prohibiting their influence on other model components such as the backbone network and expert weights.
This targeted gradient routing ensures that auxiliary losses focus solely on improving routing quality, without interfering with the optimization of the primary prediction task.

\subsection{Analysis}
To better understand the benefits of \algname, we first provide analysis on model complexity as follows.
\begin{proposition}[Complexity Analysis]
    For an $L$-layer $K$-expert \algname\ operating on a $d$-dimensional sequence of length $N$, the space complexity is $\mathcal{O}(KLd)$ and the time complexity is $\mathcal{O}(Nd^2)$.
\end{proposition}
For space complexity, \algname\ scales linearly with respect to the number of experts $K$, the number of layers $L$ and the feature dimension $d$.
This represents a significant improvement over conventional MoE frameworks~\cite{fedus2022switch,zoph2022stmoe,lepikhin2020gshard}, which typically incur quadratic parameter costs in $d$, as well as over LoRA-based MoEs~\cite{wu2024mixture,tian2024hydralora}, which introduce an additional scaling factor proportional to the LoRA rank $r$.
For time complexity, \algname\ scales linearly with the sequence length $N$ and quadratically with the feature dimension $d$, matching the computational profile of standard Transformer models.
Importantly, due to the parallel design of \algname, vertical scaling via additional MoE layers does not increase the time complexity.
This is because the computationally expensive expert transformations are executed in parallel and fused into a single transformation step after routing scores are precomputed.
In contrast, conventional vertical scaling requires sequential layer-wise computation, resulting in time complexity that grows linearly with the number of layers $L$.
Thus, \algname\ enables expressive deep scaling without incurring additional inference cost.

Besides, we analyze how the number of expert combinations scale as the the number of experts per layer $K$, the number of \algname\ layers $L$, and the number of activated experts $A$ increase.
\begin{restatable}[Model Scaling]{proposition}{scale}\label{prop:scale}
    For an $L$-layer $K$-expert \algname\ with Top-$A$ expert activation, the number of possible expert combinations is $\binom{K}{A}^L$. Specifically:
    \begin{itemize}
        \item {\ul Layer scaling}: When $K,A$ are fixed with $1\leq A < K$, expert complexity grows exponentially w.r.t. $L$.
        \item {\ul Expert scaling}: When $A,L$ are fixed, expert complexity is $\Theta(K^{AL})$, which is polynomial w.r.t. $K$.
        \item {\ul Activation scaling}: When $K,L$ are fixed, expert complexity is $\mathcal{O}\left(\left(\frac{eK}{A}\right)^{AL}\right)$, which is sublinear w.r.t. $A$.
    \end{itemize}
\end{restatable}
Proposition~\ref{prop:scale} reveals the expressive power of \algname, demonstrating that even with modest expert counts and sparse activation, hierarchical routing can yield exponentially large and diverse expert paths, which is crucial for balancing efficiency and capacity.
\section{Experiments}\label{sec:exp}
We carry out extensive experiments to answer the following research questions:
\begin{itemize}[noitemsep, topsep=0pt]
    \item\textbf{RQ1:} How effective is the proposed \algname\ on benchmark CTR datasets? (Section~\ref{sec:bench})
    \item\textbf{RQ2:} How does model performance scale with \algname\ in horinzontal and vertical directions? (Section~\ref{sec:scale})
    \item\textbf{RQ3:} To what extent does \algname\ improve the performance-efficiency tradeoff? (Section~\ref{sec:tradeoff})
    \item\textbf{RQ4:} To what extent does \algname\ benefit from different model and training framework designs? (Section~\ref{sec:study})
\end{itemize}

\subsection{Experiment Setup}\label{sec:setup}
\paragraph{Dataset} 
We evaluate on three public benchmark datasets, including AmazonElectronics\cite{he2016ups}, TaobaoAd\cite{Tianchi}, KuaiVideo\cite{li2019routing}, and MicroVideo~\cite{chen2018temporal}, which are summarized in Appendix~\ref{app:pipeline}.

\paragraph{Baseline Methods}
We consider 5 state-of-the-art Transformer-based CTR models as the backbone model, including DIN~\cite{zhou2018deep}, DIEN~\cite{zhou2019deep}, BST~\cite{chen2019behavior}, TransAct~\cite{xia2023transact} and InterFormer~\cite{zeng2024interformer}.
Besides, we compare \algname\ with 4 MoE and LoRA baselines including Non-MoE, Switch Transformer~\cite{fedus2022switch}, MoLE~\cite{wu2024mixture} and HydraLoRA~\cite{tian2024hydralora}.

\paragraph{Metrics}
We adopt three metrics to evaluate the model performance and efficiency, including:
\begin{itemize}[leftmargin=*]
    \item \texttt{AUC} provides an aggregated measure of model capacity in correctly classifying positive and negative samples across all thresholds. Higher the better.
    \item \texttt{LogLoss} (cross-entropy loss) measures the distance between the prediction $\hat{y}$ and the label $y$, and is computed as $L(y,\hat{y})\!=\!-\left(y\log\left(\hat{y}\right)+(1\!-\!y)\log\left(1\!-\!\hat{y}\right)\right)$. Lower the better.
    \item \texttt{\#Params} provides the number of parameters in the MoE module.
\end{itemize}

\subsection{Benchmark Results}\label{sec:bench}
To evaluate the effectiveness of \algname, we compare its performance against state-of-the-art MoE methods on various Transformer-based CTR models across multiple benchmark datasets.
For fair comparison, we adopt consistent model configurations on different MoE methods.
We report the results in Table~\ref{tab:bench}, from which we draw the following observations:

\begin{table*}[htbp]
\small
\caption{Benchmark results. We report AUC ($\uparrow$), LogLoss ($\downarrow$), and the parameter count of the MoE FFN layer ($\downarrow$). Blue and red cells indicate the \textcolor{blue}{best} and \textcolor{red}{second-best} performance, respectively (we exclude Non-MoE baseline for \#Params comparison as it is inherently the most parameter-efficient). For fair comparison, all non-MoE and MoE variants adopt the same configuration.}
\label{tab:bench}
\vspace{-10pt}
\begin{tabular}{@{}ll|ccl|ccl|ccl|ccl@{}}
\toprule
\multicolumn{2}{c|}{\textbf{Dataset}}              & \multicolumn{3}{c|}{\textbf{TaobaoAd}}                 & \multicolumn{3}{c|}{\textbf{AmazonElectronics}}        & \multicolumn{3}{c|}{\textbf{KuaiVideo}}                & \multicolumn{3}{c}{\textbf{MicroVideo}}               \\ \midrule
\multicolumn{2}{c|}{\textbf{Model + MoE}}          & \textbf{AUC}             & \textbf{LogLoss}         & \textbf{\#Params} & \textbf{AUC}             & \textbf{LogLoss}         & \textbf{\#Params} & \textbf{AUC}             & \textbf{LogLoss}         & \textbf{\#Params} & \textbf{AUC}             & \textbf{LogLoss}         & \textbf{\#Params} \\ \midrule
\multirow{5}{*}{\rotatebox{90}{BST}}         & Non-MoE   & 0.6486          & 0.1937          & 7.35K    & {\redcell 0.8683}    & {\redcell 0.4571}    & 24.83K   & 0.7457         & 0.4353                & 10.58K   & 0.7253          & 0.4171          & 2.13K    \\
                             & Switch    & {\redcell 0.6501}    & {\redcell 0.1935}    & 35.61K   & 0.8671          & 0.4632          & 107.78K  & 0.7458          & 0.4347          & 46.46K   & {\bluecell 0.7260} & {\redcell 0.4168}    & 12.67K   \\
                             & MoLE      & 0.6494          & 0.1937          & 26.40K   & 0.8662          & 0.4729          & 58.63K   & {\redcell 0.7462}    & {\bluecell 0.4341} & 48.00K   & {\redcell 0.7255}    & 0.4171          & 14.21K   \\
                             & HydraLoRA & 0.6484          & 0.1938          & {\redcell 16.15K}   & 0.8668          & 0.4622          & {\redcell 42.24K}   & 0.7451          & 0.4346          & {\redcell 26.50K}   & 0.7238          & 0.4191          & {\redcell 9.09K}    \\
                             & HiLoMoE   & {\bluecell 0.6505} & {\bluecell 0.1932} & {\bluecell 11.90K}   & {\bluecell 0.8699} & {\bluecell 0.4562} & {\bluecell 35.39K}   & {\bluecell 0.7463} & {\redcell 0.4343}    & {\bluecell 17.71K}   & 0.7254          & {\bluecell 0.4162} & {\bluecell 6.96K}    \\ \midrule
\multirow{5}{*}{\rotatebox{90}{DIN}}         & Non-MoE   & 0.6468          & 0.1931          & 6.27K    & 0.8762          & 0.4487          & 2.11K    & 0.7431          & 0.4386          & 4.13K    & 0.7262          & 0.4156          & 8.32K    \\
                             & Switch    & {\redcell 0.6475}    & {\redcell 0.1929}    & 31.49K   & 0.8780          & 0.4427          & 10.63K   & 0.7431          & 0.4392          & 20.68K   & {\redcell 0.7267}    & {\bluecell 0.4149} & 41.73K   \\
                             & MoLE      & {\bluecell 0.6478} & {\redcell 0.1929}    & 23.67K   & 0.8760          & 0.4440          & 10.59K   & {\redcell 0.7433}    & {\redcell 0.4374}    & 22.22K   & 0.7263          & {\redcell 0.4150}    & 29.31K   \\
                             & HydraLoRA & 0.6472          & 0.1930          & {\redcell 15.71K}   & {\bluecell 0.8795} & {\bluecell 0.4408} & {\redcell 5.95K}    & 0.7432          & {\bluecell 0.4367} & {\redcell 10.92K}   & 0.7266          & 0.4155          & {\redcell 20.06K}   \\
                             & HiLoMoE   & {\redcell 0.6475}    & {\bluecell 0.1928} & {\bluecell 13.35K}   & {\redcell 0.8789}    & {\redcell 0.4423}    & {\bluecell 4.68K}    & {\bluecell 0.7434} & 0.4380          & {\bluecell 9.16K}    & {\bluecell 0.7271} & 0.4152          & {\bluecell 17.54K}   \\ \midrule
\multirow{5}{*}{\rotatebox{90}{DIEN}}        & Non-MoE   & 0.6508          & 0.1930          & 6.27K    & 0.8803          & 0.4395          & 1.09K    & 0.7440          & 0.4366          & 4.13K    & 0.7220          & 0.4204          & 2.08K    \\
                             & Switch    & 0.6508          & 0.1926          & 31.49K   & 0.8799          & 0.4457          & 5.51K    & 0.7444    & {\redcell 0.4344}          & 20.68K   & 0.7233          & 0.4193          & 10.44K   \\
                             & MoLE      & {\redcell 0.6511}    & {\bluecell 0.1925} & 23.67K   & 0.8793          & 0.4421          & 7.00K    & {\bluecell 0.7446}          & 0.4351    & 22.22K   & {\redcell 0.7234}    & {\redcell 0.4190}    & 11.98K   \\
                             & HydraLoRA & 0.6510          & 0.1927          & {\redcell 15.71K}   & {\redcell 0.8804}    & {\redcell 0.4393}    & {\redcell 3.64K}    & 0.7436      & 0.4374   & {\redcell 10.92K}   & 0.7231          & {\bluecell 0.4185} & {\redcell 5.80K}    \\
                             & HiLoMoE   & {\bluecell 0.6513} & {\bluecell 0.1925} & {\bluecell 13.35K}   & {\bluecell 0.8806} & {\bluecell 0.4390} & {\bluecell 2.54K}    & {\bluecell 0.7446} & {\bluecell 0.4341} & {\bluecell 9.16K}    & {\bluecell 0.7235}  & 0.4192          & {\bluecell 4.68K}    \\ \midrule
\multirow{5}{*}{\rotatebox{90}{TransAct}}    & Non-MoE   & 0.6488          & 0.1933          & 24.83K   & 0.8842          & 0.4366          & 33.09K   & 0.7446          & 0.4348          & 33.31K   & 0.7266          & 0.4151          & 16.67K   \\
                             & Switch    & 0.6486          & {\redcell 0.1927}    & 105.73K  & 0.8828          & 0.4393          & 140.80K  & {\redcell 0.7472}    & 0.4343          & 141.57K  & 0.7265          & 0.4157          & 70.91K   \\
                             & MoLE      & 0.6485          & 0.1929          & 56.58K   & 0.8824          & 0.4377          & 73.22K   & 0.7469          & {\redcell 0.4333}    & 95.49K   & 0.7254          & 0.4165          & 49.41K   \\
                             & HydraLoRA & {\redcell 0.6491}    & {\redcell 0.1927}    & {\redcell 40.19K}   & {\redcell 0.8851}    & {\bluecell 0.4340} & {\redcell 52.74K}   & 0.7470          & 0.4349          & {\redcell 60.67K}   & {\redcell 0.7280}    & {\redcell 0.4137}    & {\redcell 30.98K}   \\
                             & HiLoMoE   & {\bluecell 0.6495} & {\bluecell 0.1926} & {\bluecell 33.34K}   & {\bluecell 0.8860} & {\redcell 0.4351}    & {\bluecell 44.22K}   & {\bluecell 0.7473} & {\bluecell 0.4331} & {\bluecell 46.43K}   & {\bluecell 0.7281} & {\bluecell 0.4130} & {\bluecell 23.39K}   \\ \midrule
\multirow{5}{*}{\rotatebox{90}{InterFormer}} & Non-MoE   & 0.6515          & 0.1926          & 1.07K    & 0.8829          & 0.4297          & 2.11K    & 0.7421          & 0.4369          & 2.13K    & 0.7165          & 0.4173          & 2.13K    \\
                             & Switch    & 0.6513          & 0.1926          & 6.40K    & {\redcell 0.8854}    & 0.4299          & 12.67K   & {\bluecell 0.7451} & {\bluecell 0.4351} & 12.67K   & 0.7167          & 0.4169          & 12.67K   \\
                             & MoLE      & {\redcell 0.6523}    & {\redcell 0.1925}    & 7.94K    & 0.8827          & 0.4325          & 12.68K   & 0.7430          & 0.4361          & 14.21K   & 0.7167          & 0.4176          & 14.21K   \\
                             & HydraLoRA & 0.6521          & {\redcell 0.1925}    & {\redcell 4.86K}    & 0.8845          & {\redcell 0.4296}    & {\redcell 8.58K}    & 0.7434          & 0.4366          & {\redcell 9.09K}    & {\redcell 0.7169}    & {\redcell 0.4166}    & {\redcell 9.09K}    \\
                             & HiLoMoE   & {\bluecell 0.6541} & {\bluecell 0.1924} & {\bluecell 3.57K}    & {\bluecell 0.8874} & {\bluecell 0.4279} & {\bluecell 6.82K}    & {\redcell 0.7450}    & {\redcell 0.4359}    & {\bluecell 6.96K}    & {\bluecell 0.7178} & {\bluecell 0.4161} & {\bluecell 6.96K}    \\ \bottomrule
\end{tabular}
\end{table*}

\textbf{(1) \algname\ achieves state‑of‑the‑art performance across all models and datasets.}
Compared to the non‑MoE baseline, \algname\ attains the highest AUC in 16 out of 20 cases and the second highest in three additional cases. For LogLoss, it ranks best in 13 cases and second-best in four.
These results show that \algname\ consistently enhances prediction performance with minimal additional complexity, validating its robustness on diverse settings.

\textbf{(2) Compared to the non-MoE baseline, \algname\ achieves consistent gains in prediction quality while maintaining compact model size.}
Across all CTR models and datasets, \algname\ achieves an average improvement of 0.15\% on AUC and 0.13\% on LogLoss.
This performance improvement is achieved with only a modest increase of 6.17K parameters on average, which includes the low-rank expert weights and router components.
These findings highlight \algname's ability to provide meaningful performance improvements with limited overhead, making it well-suited for deployment in resource-constrained environments.

\textbf{(3) Relative to existing state-of-the-art MoE baselines (Switch Transformer~\cite{fedus2022switch}, MoLE~\cite{wu2024mixture}, HydraLoRA~\cite{tian2024hydralora}), \algname\ consistently shows superior efficiency and effectiveness.}
On average, it improves AUC by 0.08\% and reduces LogLoss by 0.10\% compared to the best competing MoE variant (HydraLoRA).
At the same time, \algname\ reduces parameter count by an average of 4.04K, which is equivalent to a 21.0\% reduction relative to the most parameter-efficient MoE competitor (HydraLoRA).
The improved accuracy can be attributed to \algname's combinatorially diverse expert compositions enabled by hierarchical routing, while the efficiency gains stem from its low-rank expert formulation.
These results show that \algname\ offers a much more efficient MoE configuration with on par or even enhanced accuracy.

\subsection{Model Scaling}\label{sec:scale}
Model scaling, improving performance as capacity grows with minimal saturation, is a key property for CTR models. 
We evaluate \algname\ under horizontal (\#experts) and vertical (\#layers) scaling.

\subsubsection{Horizontal Scaling}
\begin{figure*}[htbp]
    \centering
    \begin{subfigure}{.48\textwidth}
        \centering
        \includegraphics[width=\textwidth, trim=10 0 10 0, clip]{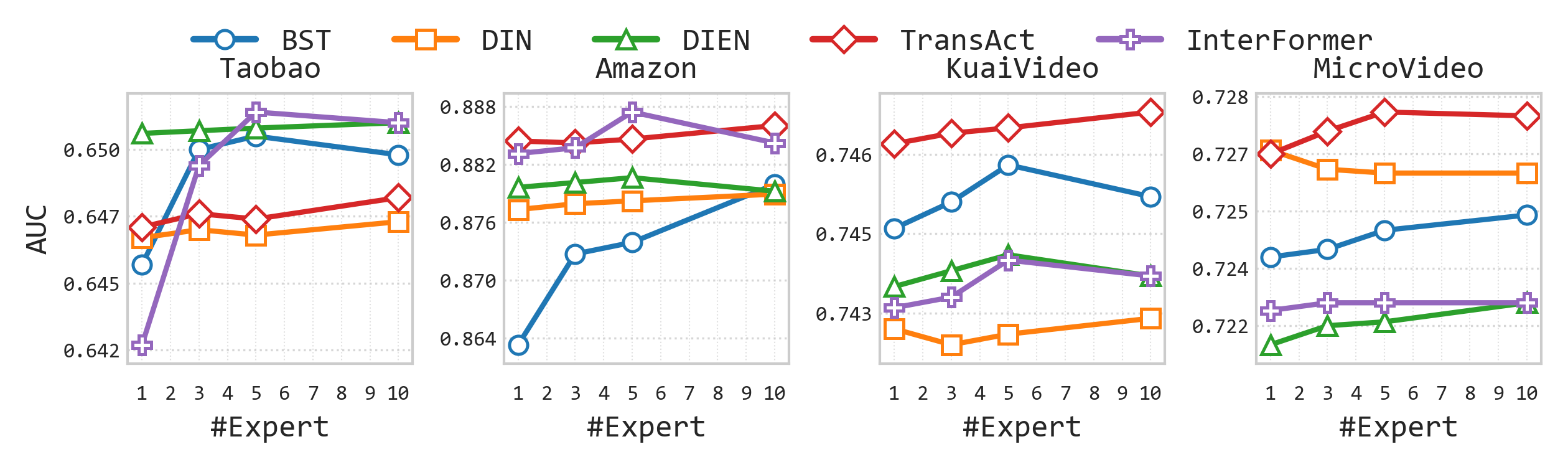}
        \vspace{-15pt}
        \caption{Horizontal scaling on AUC.}
        \label{fig:expert_auc}
    \end{subfigure}
    \hfill
    \begin{subfigure}{.48\textwidth}
        \centering
        \includegraphics[width=\textwidth, trim=10 0 10 0, clip]{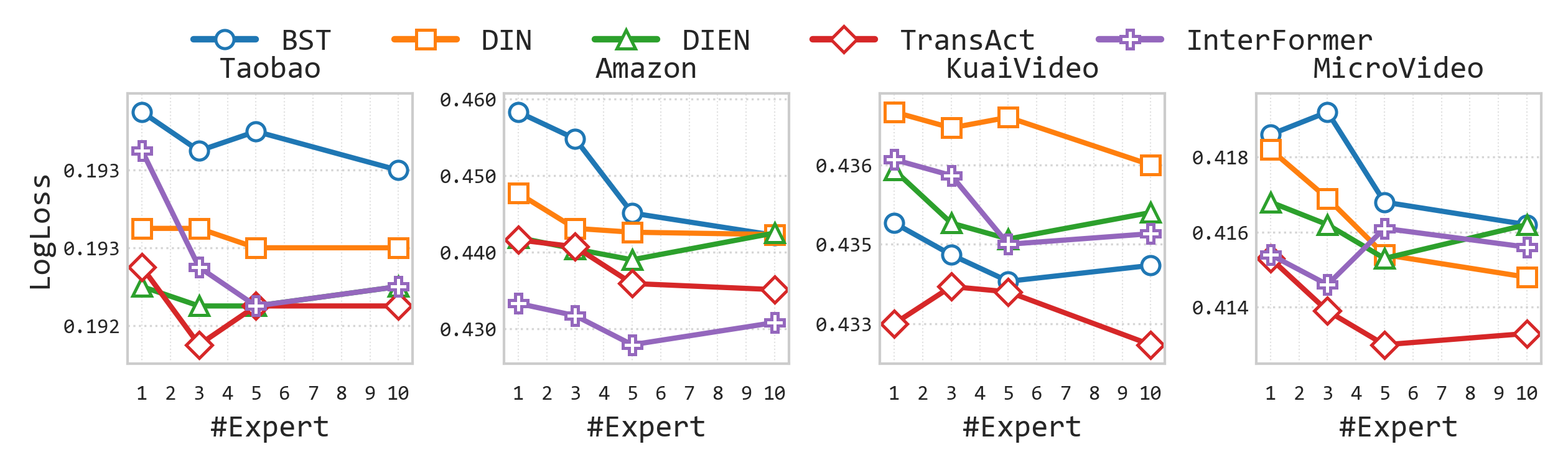}
        \vspace{-15pt}
        \caption{Horizontal scaling on LogLoss.}
        \label{fig:expert_logloss}
    \end{subfigure}
    \vspace{-10pt}
    \caption{Horizontal scaling on \#experts.}
    \vspace{-10pt}
    \label{fig:expert}
\end{figure*}

We first evaluate the horizontal scaling by varying the number of experts in \{1,3,5,10\}.
As shown in Figure~\ref{fig:expert}, \textbf{increasing experts generally improves AUC and reduces LogLoss, confirming the effectiveness of horizontal scaling.}

Across the majority of models and datasets, we observe a monotonic or near-monotonic gain in AUC as the number of experts increases.
Similarly, we observe a steady decline in LogLoss as the number of experts increases, indicating not only improved ranking performance but also better-calibrated probability estimates.
This trend shows enhanced model expressiveness brought by expert diversity, leading to more accurate prediction.
For example, increasing the number of experts within a moderate range (e.g., 1 to 5) results in a clear AUC lift and LogLoss decline.
When more experts are involved, the performance either continues to improve or plateaus afterward, suggesting that a small number of experts is already sufficient to bring meaningful gains, especially when paired with efficient routing and training strategies.
In addition, we observe no significant degradation when more experts are added, demonstrating the robustness against overparameterization.

\subsubsection{Vertical Scaling}
\begin{figure*}[htbp]
    \centering
    \begin{subfigure}{.48\textwidth}
        \centering
        \includegraphics[width=\textwidth, trim=10 0 10 0, clip]{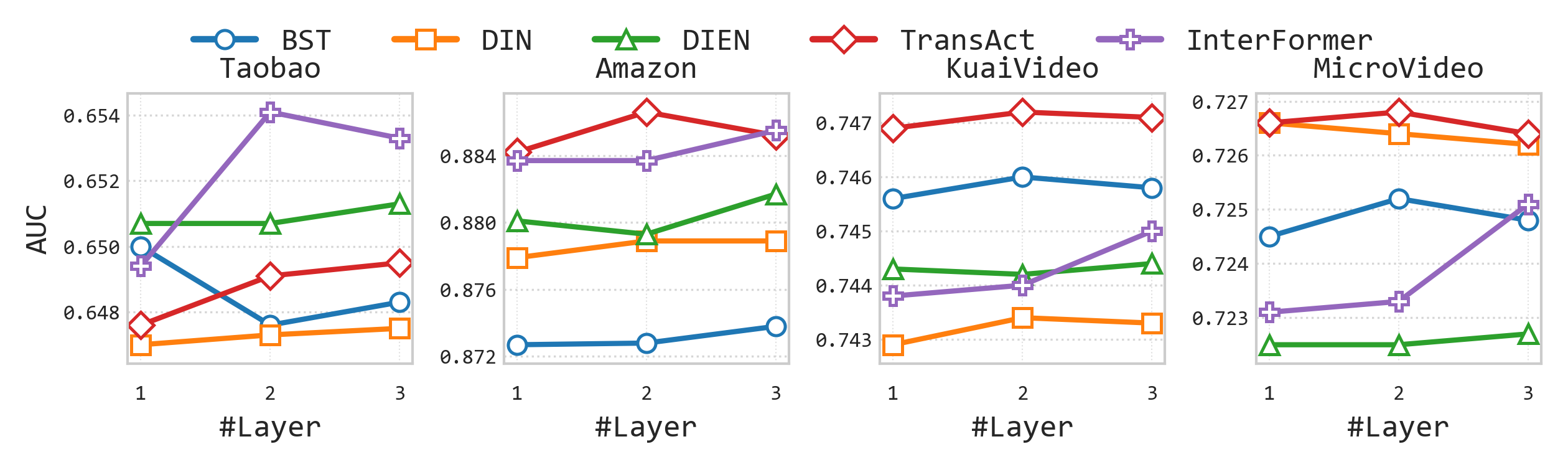}
        \vspace{-15pt}
        \caption{Vertical scaling on AUC.}
        \label{fig:layer_auc}
    \end{subfigure}
    \hfill
    \begin{subfigure}{.48\textwidth}
        \centering
        \includegraphics[width=\textwidth, trim=10 0 10 0, clip]{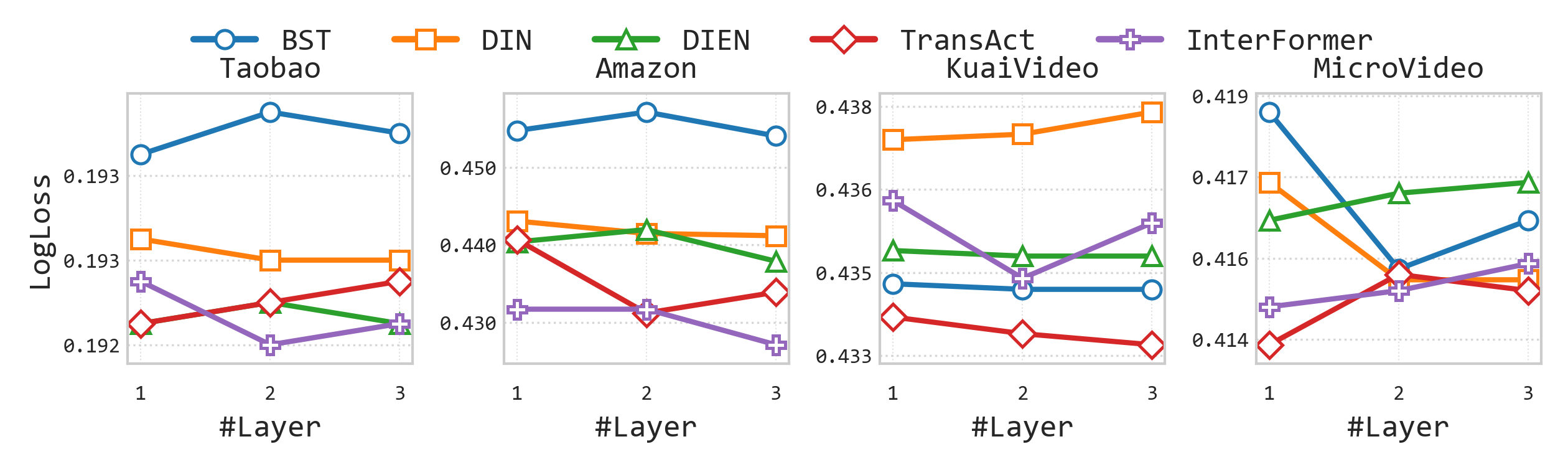}
        \vspace{-15pt}
        \caption{Vertical scaling on LogLoss.}
        \label{fig:layer_logloss}
    \end{subfigure}
    \vspace{-10pt}
    \caption{Vertical scaling on \#layers.}
    \vspace{-5pt}
    \label{fig:layer}
\end{figure*}

We also evaluate the vertical scaling of \algname\ by varying the number of MoE layers in \{1,2,3\}.
The results in Figure~\ref{fig:layer} suggest that \textbf{modest vertical scaling leads to incremental improvements in AUC and LogLoss, though the gains are generally smaller compared to horizontal scaling.}

Across most models and datasets, increasing the number of MoE layers from one to two consistently leads to noticeable performance gains, confirming that deeper expert composition can enhance representational capacity.
While gains may taper or fluctuate slightly beyond two layers, this behavior reflects the need for architectural balance, neither too deep nor too shallow, rather than an inherent limitation of depth.
Overall, these results suggest that vertical scaling is a promising and complementary direction within the broader \algname\ framework, enabling the construction of more expressive and efficient CTR models when combined with expert-level scaling.

Among the datasets, MicroVideo exhibits clearer vertical scaling benefits, especially in models like InterFormer and TransAct.
This may be attributed to its long behavioral sequences and relatively compact feature space, which allow deeper transformations to be effectively stacked without overfitting.
In contrast, on datasets such as TaobaoAd and KuaiVideo, the benefits of adding additional layers are marginal, possibly due to their shorter sequences or simpler patterns, which may not require deep hierarchical modeling.

\textit{Remark.} We observe that vertical scaling yields limited or even negative gains compared to horizontal scaling. This is because stacking more \algname\ layers increases dependency depth and routing uncertainty, leading to redundant transformations and optimization instability.
In contrast, horizontal scaling expands expert diversity at a fixed depth, aligning better with the heterogeneous yet shallow nature of user–item interactions in CTR tasks.

\subsection{Performance-Efficiency Tradeoff}\label{sec:tradeoff}
\begin{figure}[t]
    \centering
    \includegraphics[width=\linewidth]{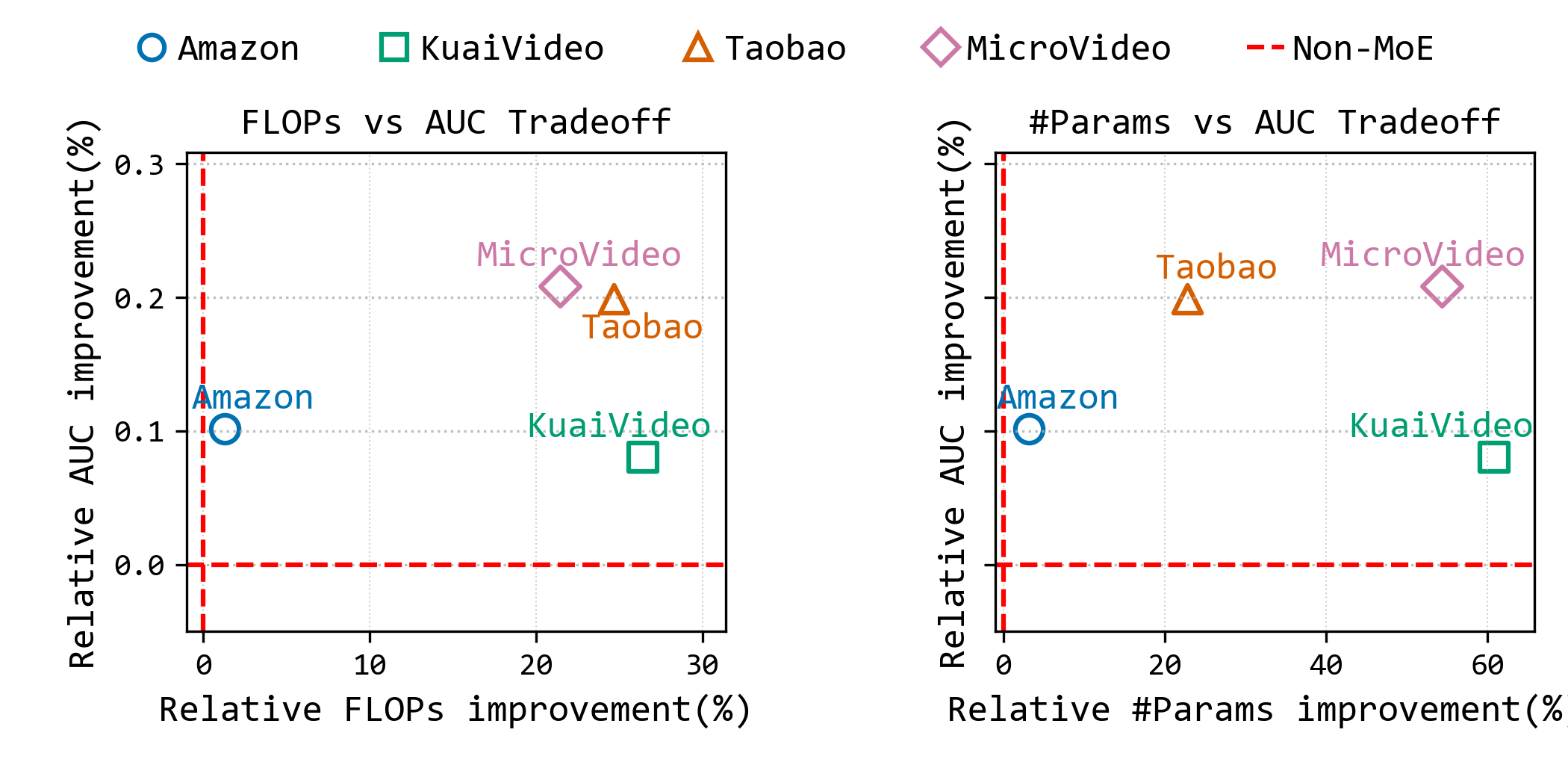}
    \vspace{-10pt}
    \caption{Performance-Efficiency Tradeoff. We report the relative improvements in AUC and FLOPs achieved by \algname\ on the InterFormer model, compared to its non-MoE counterpart. To ensure a fair comparison under optimal conditions, the baseline non-MoE model is scaled up to a larger configuration to reach the best AUC.}
    \vspace{-15pt}
    \label{fig:tradeoff}
\end{figure}

\begin{figure*}[t]
    \centering
    \begin{subfigure}{.24\textwidth}
        \centering
        \includegraphics[width=\textwidth, trim=10 0 10 0, clip]{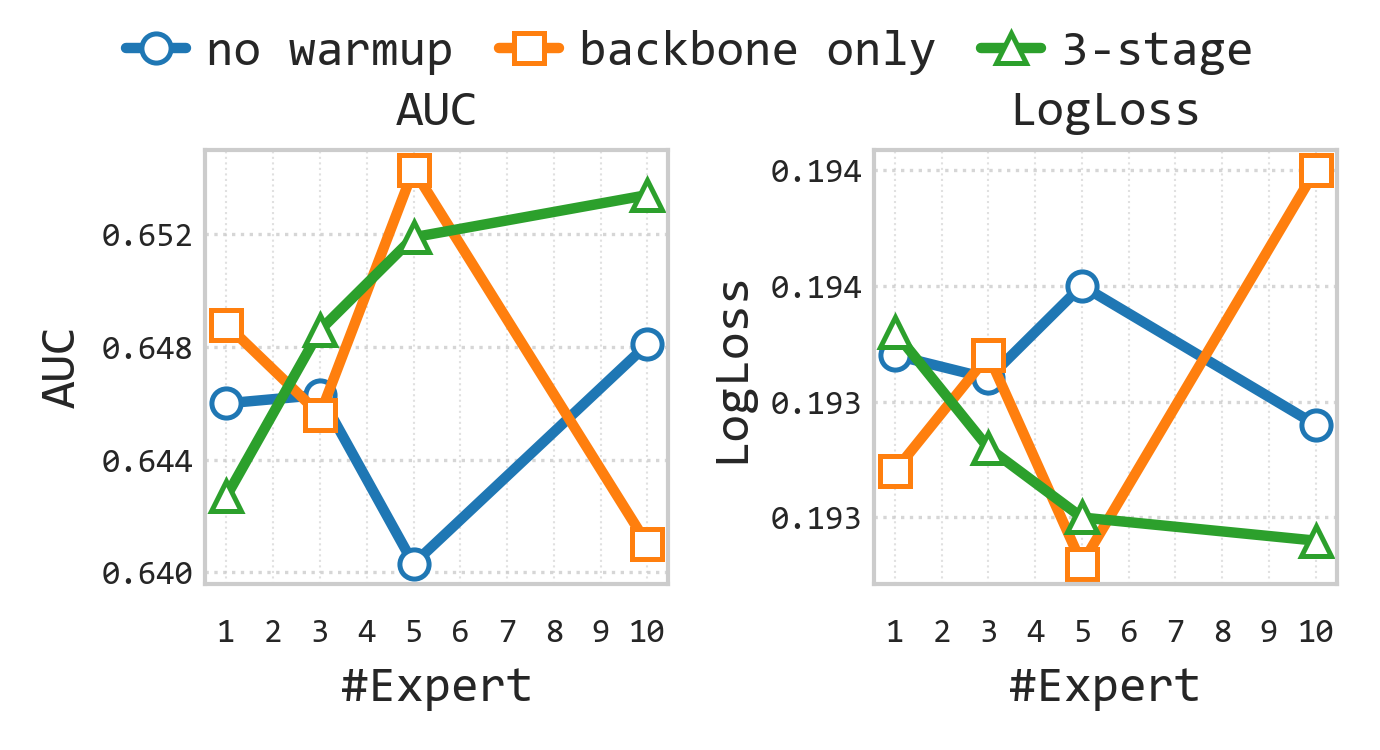}
        \vspace{-15pt}
        \caption{Study on model warmup.}
        \label{fig:exp-warm}
    \end{subfigure}
    \hfill
    \begin{subfigure}{.24\textwidth}
        \centering
        \includegraphics[width=\textwidth, trim=10 0 10 0, clip]{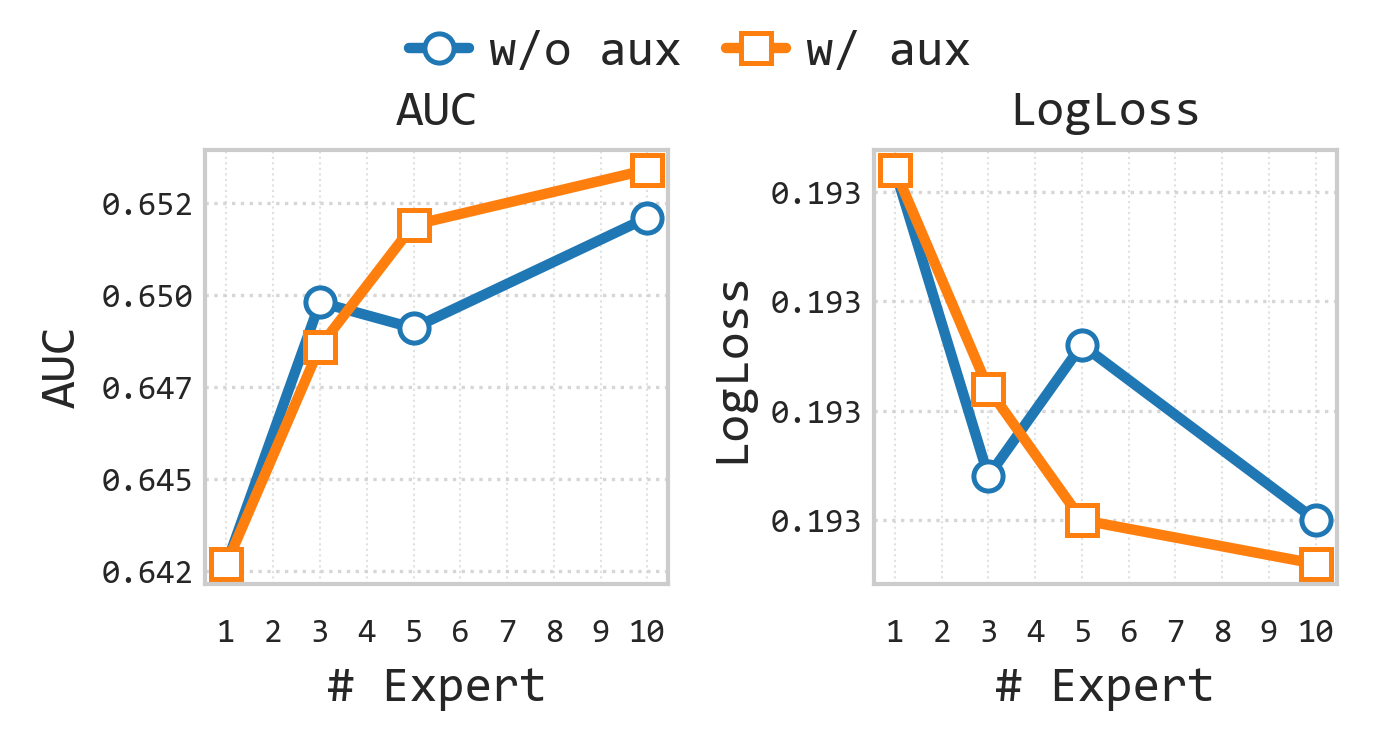}
        \vspace{-15pt}
        \caption{Study on auxiliary losses.}
        \label{fig:exp-aux}
    \end{subfigure}
    \hfill
    \begin{subfigure}{.24\textwidth}
        \centering
        \includegraphics[width=\textwidth, trim=10 0 10 0, clip]{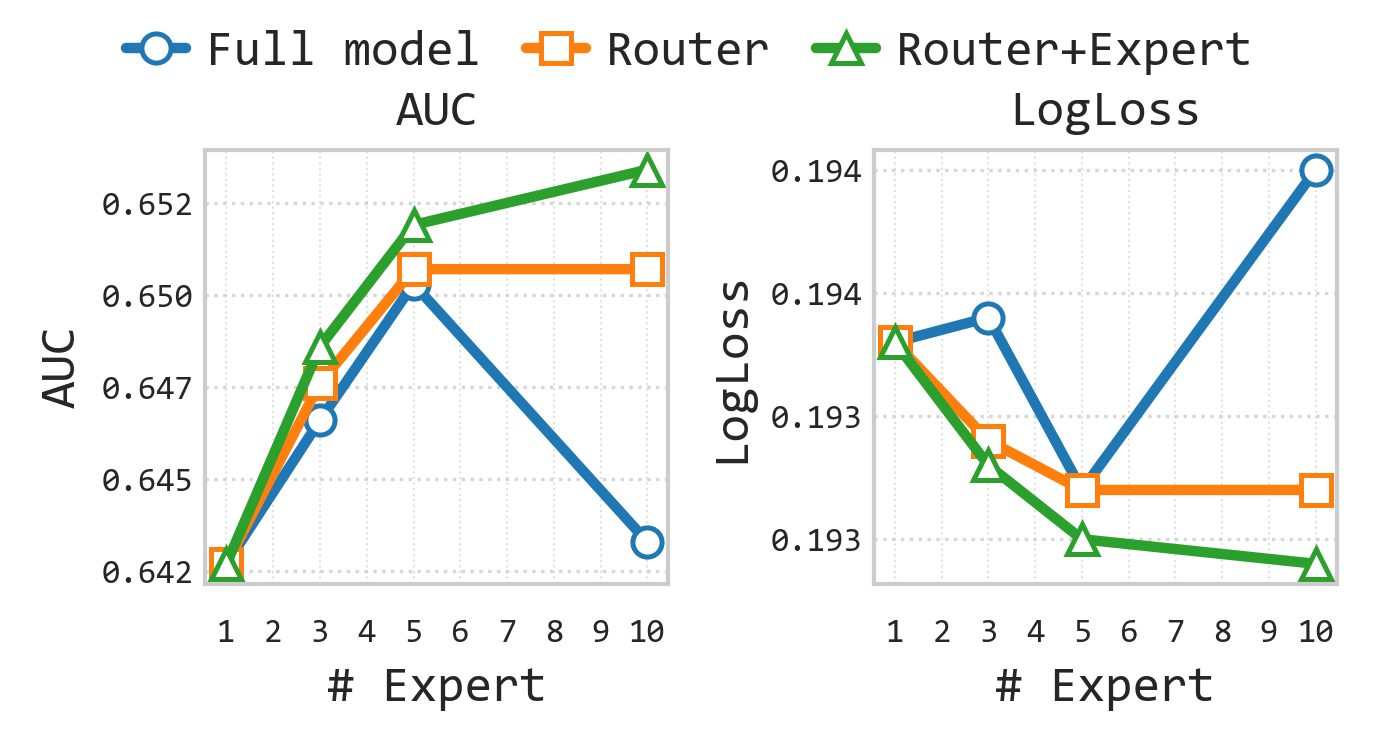}
        \vspace{-15pt}
        \caption{Study on aux loss gradients.}
        \label{fig:exp-grad}
    \end{subfigure}
    \hfill
    \begin{subfigure}{.24\textwidth}
        \centering
        \includegraphics[width=\textwidth, trim=10 0 10 0, clip]{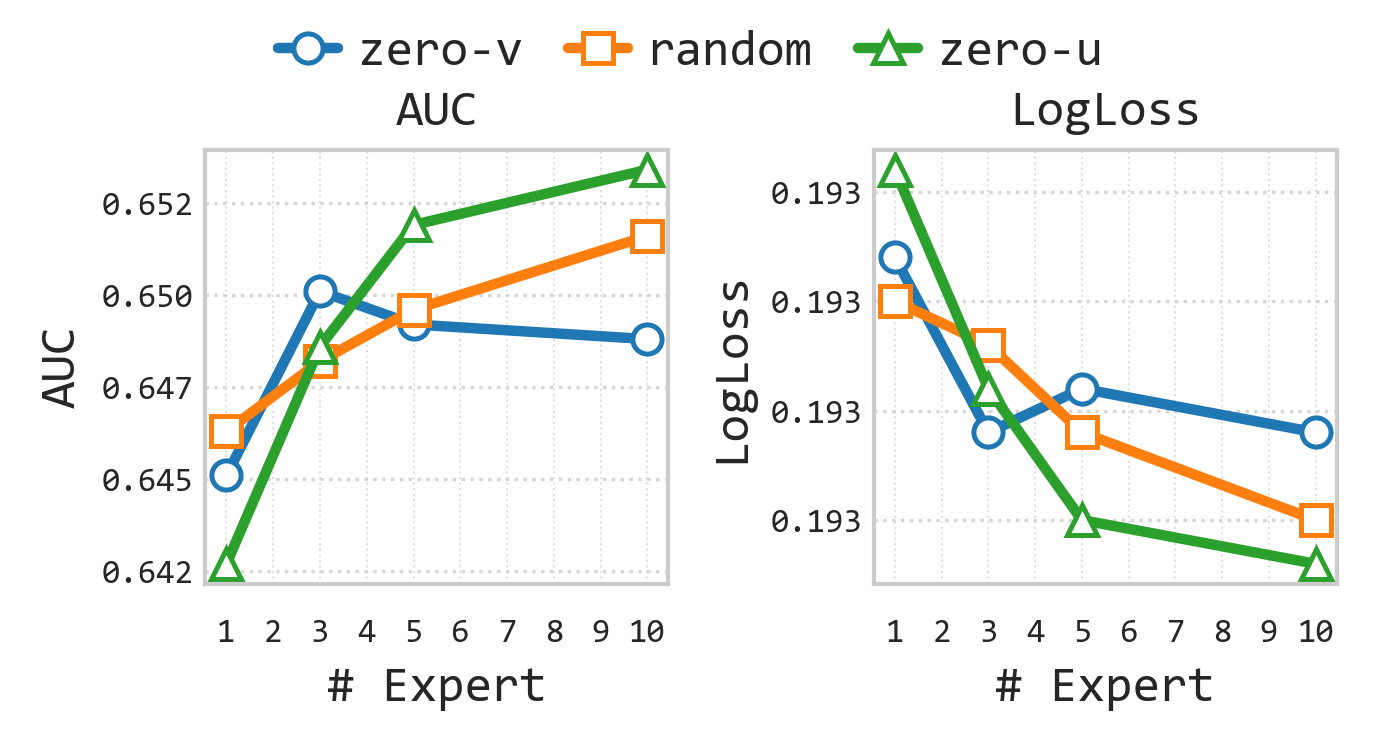}
        \vspace{-15pt}
        \caption{Study on model initialization.}
        \label{fig:exp-init}
    \end{subfigure}
    \vspace{-10pt}
    \caption{Studies on \algname. We compare model performance with (a) different training strategy, (b) auxiliary losses, (c) auxiliary loss gradient flows, and (d) model initialization. In general, \algname\ trained under 3-stage training framework, augmented with auxiliary losses and initialized with $\textit{zero-u}$ initialization yields the best performance.}
    \vspace{-5pt}
    \label{fig:exp-study}
\end{figure*}

We evaluate the performance-efficiency tradeoff of \algname\ by comparing it with the non-MoE baseline.
In general, we compute the relative improvements in AUC against both computational cost (FLOPs) and model size (\#Params), using the non-MoE baseline as a reference point.
We adopt the most state-of-the-art InterFormer model as the baseline, and fine-tune the non-MoE InterFormer configuration to achieve the best AUC.
The non-MoE AUC, \#params and FLOPs are further used as the reference points to compute the relative improvements, as shown in Figure~\ref{fig:tradeoff}.

\textbf{(1) AUC vs FLOPs (left): \algname\ consistently achieves better performance and efficiency compared to the non-MoE baseline}.
Notably, on datasets like Taobao and MicroVideo, \algname\ reduces FLOPs by over 20\% while still improving AUC by approximately 0.2\%, highlighting its strong scalability and practical value in resource-constrained environments.
Even on Amazon, where the reduction in FLOPs is minimal, we still observe a clear AUC gain, suggesting that HiLoMoE can deliver performance benefits even without sacrificing efficiency.

\textbf{(2) AUC vs \#Params (right): \algname\ exhibits substantial reductions in parameter count across all datasets, while achieving consistent AUC improvements}.
On KuaiVideo, \algname\ achieves a 60\% reduction in \#Params compared to the non-MoE baseline.
This demonstrates the effectiveness of the LoRA-based expert design in enabling parameter-efficient modeling without compromising accuracy.

Overall, the tradeoff figures illustrate that \algname\ offers a highly favorable balance between accuracy and efficiency. 
By leveraging low-rank expert parameterization and sparse hierarchical routing, \algname\ successfully reduces resource consumption while enhancing model performance, making it well-suited for large-scale CTR prediction systems in both training and deployment scenarios.

\subsection{Studies}\label{sec:study}
In this section, we carry out studies on model warmup (Section~\ref{sec:exp-warmup}), auxiliary losses (Section~\ref{sec:exp-aux}) and model initialization (Section~\ref{sec:exp-init}).

\subsubsection{On the effect of model warmup}\label{sec:exp-warmup}
We first study the effect of model warmup, and results are shown in Figure~\ref{fig:exp-warm}.
We consider three variants: (1) \emph{no warmup}, where the whole model is trained directly without warmup, (2) \emph{backbone only}, where only the backbone model is warmed up while the experts are not, and (3) \emph{3-stage}, where we perform the three-stage training illustrated in Figure~\ref{fig:train}.

When the numbers of experts increases, \emph{no warmup} leads to unstable performance, with AUC fluctuating and LogLoss remaining relatively high.
\emph{Backbone only} provides partial improvement, but its scaling performance is less stable with performance degrading as the number of experts grows.
In contrast, the proposed \emph{3-stage} warmup achieves the most stable scaling behavior, as AUC increases steadily, while LogLoss decreases smoothly.
These results highlight that progressive warmup not only stabilizes optimization but also allows the experts to be more effectively integrated with the backbone, thereby unlocking the benefits of expert scaling.

\subsubsection{On the effect of auxiliary losses}\label{sec:exp-aux}
We then compare model performance with and without auxiliary losses, and results are shown in Figure~\ref{fig:exp-aux}.
Overall, incorporating auxiliary losses consistently improves both AUC and LogLoss.
Without auxiliary losses, the model exhibits less stable scaling as the number of experts increases, with AUC and LogLoss fluctuating when the number of experts increases from 3 to 10.
By contrast, applying auxiliary losses yields a smoother scaling curve: AUC increases steadily, while LogLoss decreases monotonically.
This suggests that auxiliary objectives effectively regularize the routing mechanism, preventing expert collapse and promoting more balanced utilization. As a result, the experts contribute more complementary knowledge, leading to measurable gains in both predictive accuracy and calibration.

Besides, we study how different gradient flows of auxiliary losses affect model performance.
We consider three different variants, including (1) \emph{full model}, where auxiliary losses affect the training of the whole model, (2) \emph{router}, where the auxiliary losses only affect the router update, and (3) \emph{router+expert}, where the auxiliary losses affect both router and expert update.
We observe that letting the auxiliary losses propagate through the \emph{full model} leads to unstable behavior: while model performance peaks with 5 experts, it quickly drops when the number of experts increases.
Restricting the auxiliary losses to update only the \emph{router} yields a more stable trend, with AUC reaching 0.650 and LogLoss stabilizing around 0.193, but further gains are limited as expert scaling increases.
In contrast, applying auxiliary losses to both the \emph{router+expert} consistently delivers the best performance: AUC steadily improves beyond 0.652 at 10 experts, while LogLoss decreases smoothly to below 0.193.
These results confirm that auxiliary objectives are most effective when guiding both routing and expert specialization, while constraining their gradient flow away from the backbone avoids interference with the main prediction task.

\subsubsection{On the effect of expert initialization}\label{sec:exp-init}
Finally, we study the effect of different expert initialization strategies, and results are shown in Figure~\ref{fig:exp-init}.
We consider three variants: (1) \emph{zero-v}, where $\Vl$ is zero-initialized and $\Ul$ is randomly initialized, (2) \emph{random}, where both $\Vl$ and $\Ul$ are randomly initialized, and (3) \emph{zero-u}, where $\Ul$ is zero-initialized and $\Vl$ is randomly initialized.

Among the three variants, \emph{zero-v} leads to the weakest performance: AUC quickly saturates and LogLoss remains relatively high.
This degradation is mainly because $\Vl$ is used to compute the average expert representation and input queries for hierarchical routing in Eq.~\eqref{eq:query}, and zero-initializing $\Vl$ thus diminishes expert diversity and limits expressiveness in the early training stage.
In contrast, \emph{random} initialization provides a stronger baseline, with AUC gradually improving and LogLoss decreasing smoothly as the number of experts increases. 
The best performance is achieved by \emph{zero-u}, where the input direction $\Ul$ is zero-initialized while $\Vl$ is randomly initialized: AUC surpasses 0.652 at 10 experts, and LogLoss consistently decreases to below 0.193. 
These results suggest that freezing the input direction at initialization suppresses noisy expert activations in early training, while allowing the output direction to adapt flexibly, thereby stabilizing optimization and improving convergence.
\section{Conclusion}\label{sec:con}
In this paper, we present \algname, a hierarchical LoRA MoE framework designed to achieve holistic model scaling for Transformer-based CTR prediction.
By combining low-rank expert parameterization with hierarchical routing, \algname\ supports efficient horizontal and vertical scaling while maintaining parameter efficiency and parallelizable inference.
To ensure stable optimization, we introduce a three-stage training framework with auxiliary losses to encourage expert diversity and balanced utilization.
Extensive experiments across multiple benchmark datasets and backbone architectures demonstrate that \algname\ consistently improves predictive accuracy while reducing computational cost, achieving an average of 0.20\% gains in AUC and an average 18.5\% reduction in FLOPs compared to non-MoE baselines.
These results highlight the promise of structured and parameter-efficient MoE designs for advancing CTR prediction under practical efficiency constraints.

\newpage
\bibliographystyle{ACM-Reference-Format}
\balance
\bibliography{main}

\newpage
\appendix
\section*{Appendix}

\section{Proof}
\subsection{Proof of Proposition~\ref{prop:scale}}
\scale*

\begin{proof}
    Each \algname\ layer with Top-$A$ expert activation generates $\binom{K}{A}$ possible combinations, and stacking $L$ layers further enlarges the number of combinations to $\binom{K}{A}^L$.

    When $K,A$ are fixed, it is obvious to show that the expert complexity grows exponential w.r.t. $L$.

    Besides, we can bound $\binom{K}{A}$ as follows~\cite{cormen2022introduction}:
    \begin{equation*}
        \left(\frac{K}{A}\right)^{A}\leq \binom{K}{A}\leq \left(\frac{eK}{A}\right)^A
    \end{equation*}
    
    Therefore, when $A, L$ are fixed, the expert complexity can be expressed by $N=\Theta(K^{AL})$, which is polynomial w.r.t. $K$.
    
    When $K,L$ are fixed, according to the Stirling approximation~\cite{cormen2022introduction}
    \begin{equation*}
        \ln\binom{K}{A}=A\ln{\frac{eK}{A}}+\mathcal{O}(\ln A),
    \end{equation*}
    which results in an expert complexity of $\mathcal{O}\left(\frac{eK}{A}\right)^{AL}$ that is sublinear w.r.t. $A$.
\end{proof}

\section{Additional Experiments}

We visualize the routing score along the training process with and without auxiliary losses in Figure~\ref{fig:app-routing}.
We adopt Transact as the backbone model, and average the routing score in each 1,000 batches.
Darker color indicates the expert is more frequently activated.
\begin{figure}[h]
    \centering
    \begin{subfigure}{.23\textwidth}
        \centering
        \includegraphics[width=\textwidth, trim=0 0 10 10, clip]{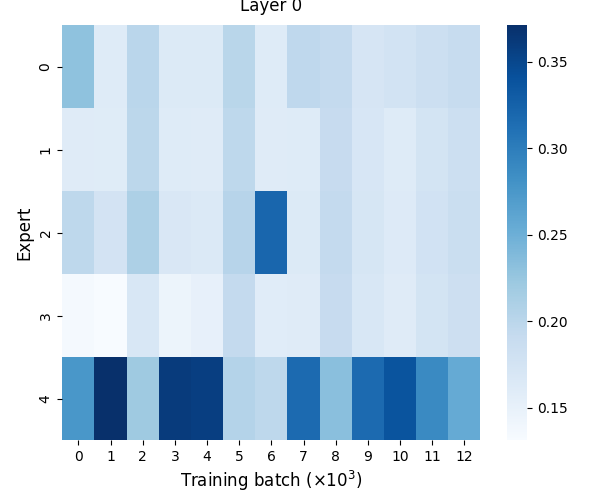}
        \vspace{-18pt}
        \caption{Routing score w/o aux losses.}
        \label{fig:app-routing1}
    \end{subfigure}
    \begin{subfigure}{.23\textwidth}
        \centering
        \includegraphics[width=\textwidth, trim=0 0 10 10, clip]{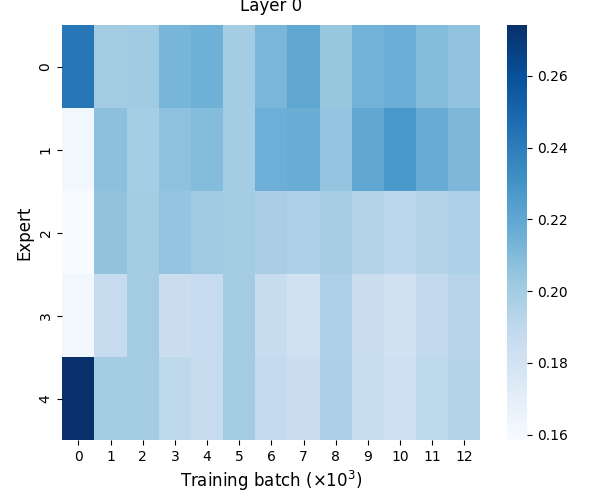}
        \vspace{-18pt}
        \caption{Routing score w/ aux losses.}
        \label{fig:app-routing2}
    \end{subfigure}
    \vspace{-10pt}
    \caption{Routing score along the training process with and without auxiliary losses.}
    \label{fig:app-routing}
\end{figure}

Without auxiliary losses (Figure~\ref{fig:app-routing1}), the router exhibits clear mode collapse: load concentrates on a single expert for long stretches, sharply reducing expert diversity and, in the extreme, degenerating to a non-MoE model. 
In contrast, adding auxiliary losses (Figure~\ref{fig:app-routing2}) yields markedly more even expert utilization over time. 
This prevents collapse, promotes expert specialization, and better exploits MoE diversity for personalized recommendation.

\section{Experiment Pipeline}\label{app:pipeline}
\subsection{Reproducibility}
We adopt the FuxiCTR library~\cite{zhu2020fuxictr} and BARS evaluation framework~\cite{zhu2022bars} for benchmark experiments, and conduct all experiments on NVIDIA A100 GPU.

For model configuration, we adopt the Top-1 activation for MoE.
We fine-tune the number of \algname\ layers in \{1,2,3\}, the number of experts in \{1,3,5,10\}, and the embedding dimensions in \{32,64,128\}.

For model training, an Adam~\cite{kingma2014adam} optimizer with a learning rate scheduler is adopted for model optimization, where the initial learning rate is tuned in \{1e-1, 1e-2, 1e-3\}.
We use a batch size of 2048, and train up to 100 epochs with early stop.
For model warmup, we warm up the backbone model for 10,000 batches, each \algname\ layer for 5,000 batches, and fine-tune the full model till early stop.

For benchmark datasets, we adopt the default processing and configurations in BARS~\cite{zhu2022bars}.


\subsection{Dataset}
Dataset statistics are summarized in Table~\ref{tab:data}, and detailed descriptions are provided as follows:

\begin{table}[htbp]
\caption{Dataset Statistics.}
\label{tab:data}
\begin{tabular}{@{}llll@{}}
\toprule
Dataset           & \#Samples & \#Feat. (Seq/Non-Seq) & Seq Length \\ \midrule
Amazon            & 3.0M      & 5 (2/3)               & 100                            \\
TaobaoAd          & 25.0M     & 20 (3/17)             & 50                             \\
KuaiVideo         & 13.7M     & 6 (2/4)               & 100                            \\ 
MicroVideo         & 12.7M     & 6 (2/4)               & 100                            \\ \bottomrule
\end{tabular}
\end{table}

\begin{itemize}
    \item \textbf{AmazonElectronics}~\cite{he2016ups} is a large-scale dataset containing product reviews and metadata collected from Amazon’s electronics category. It includes 192,403 users, 63,001 items, 801 categories, and a total of 1,689,188 interaction samples.
    The dataset features both non-sequential features, e.g., user ID, item ID, and item category, and sequence features, e.g., interacted items and corresponding categories, truncated or padded to a fixed length of 100.
    Following standard benchmark configurations, the dataset is split into 2.60 million training samples and 0.38 million testing samples.
    \item \textbf{TaobaoAds}~\cite{Tianchi} is a large-scale dataset consisting of 26 million ad click-through records collected over 8 consecutive days on Taobao, randomly sampled from 1,140,000 users.
    The dataset contains both non-sequential and sequential features.
    Non-sequential features include item-related attributes such as ad ID, category, and price, as well as user-related attributes such as user ID, gender, and age.
    Sequence features capture users’ historical interactions, including sequences of item brands, categories, and behavior types, each truncated or padded to a fixed length of 50. 
    Following the public benchmark configuration, we use 22.0 million samples for training and 3.1 million samples for testing.
    \item \textbf{KuaiVideo}~\cite{li2019routing} is a large-scale video interaction dataset comprising 3,239,534 user-video interactions from 10,000 users. 
    It includes both non-sequence and sequence features.
    The non-sequence features consist of user ID, video ID, and pre-extracted visual embeddings of the videos.
    The sequence features capture user behavior histories, such as click, like, and skip interactions, with a fixed sequence length of 100.
    Following the public benchmark split, the dataset is divided into 10.9 million training samples and 2.7 million testing samples.
    \item \textbf{MicroVideo}~\cite{chen2018temporal} is a large-scale dataset collected from a popular micro-video sharing platform, comprising 12,737,619 user–video interactions from 10,986 users and 1,704,880 unique micro-videos.
    Each micro-video is associated with visual features extracted from its cover image using the Inception-v3 model, and a manually assigned category from a predefined set of 512 mutually exclusive categories. 
    Interaction records include user ID, video ID, and timestamp, and cover both positive (clicked) and negative (skipped) feedback.
    Following the public benchmark split, the dataset consists of 8.9 million training samples and 3.8 million testing samples.
\end{itemize}

\end{document}